\documentclass[letterpaper,12pt]{article}

\usepackage{fullpage}

\usepackage{authblk}
\usepackage{amsmath,amsthm,amssymb}
\usepackage{url}
\usepackage{graphicx}
\usepackage{subfig}
\usepackage{cite}
\usepackage{comment}

\title{Convergence of Nearest Neighbor Pattern Classification with Selective Sampling}
\author{Shaun N.~Joseph\footnote{Correspondence: snjoseph@gmail.com}}
\author{Seif Omar Abu Bakr}
\author{Gabriel Lugo}
\affil{mZeal Communications, Inc. \\ Fitchburg MA 01420}

\newtheorem{thm}{Theorem}
\newtheorem*{cor}{Corollary}
\newtheorem{lem}[thm]{Lemma}

\newtheorem*{bclem2}{Second Borel-Cantelli Lemma}
\newtheorem{prp}[thm]{Proposition}

\theoremstyle{definition}

\theoremstyle{remark}
\newtheorem{rem}{Remark}

\newcommand{\thmref}[1]{Theorem~\ref{#1}}
\newcommand{\secref}[1]{\S\ref{#1}}
\newcommand{\lemref}[1]{Lemma~\ref{#1}}
\newcommand{\prpref}[1]{Proposition~\ref{#1}}

\newcommand{\figref}[1]{Figure~\ref{#1}}
\newcommand{\remref}[1]{Remark~\ref{#1}}

\DeclareMathOperator{\supp}{supp}
\DeclareMathOperator{\mode}{mode}
\DeclareMathOperator{\modefreq}{modefreq}

\begin{document}
\maketitle

\section{Introduction}
\label{sec:intro}

In the panoply of pattern classification techniques, few enjoy the
intuitive appeal and simplicity of the \emph{nearest neighbor rule}:
given a set of samples in some domain space whose value under some function
is known, estimate the function anywhere in the domain by giving the value
of the nearest sample (relative to some metric).
More generally, one may use the modal value of the $m$ nearest samples,
where $m \geq 1$ is some fixed integer constant, although $m=1$ is known to
be admissible in the sense that there is no $m > 1$ that is asymptotically
superior in terms of prediction error \cite{cover-hart}.

The nearest neighbor rule is a \emph{nonparametric} technique; that is,
it does not make any assumptions about the character of the underlying
function (eg, linearity) and proceed to estimate parameters modulo this
assumption (eg, slope and intercept). Furthermore, it is extremely general,
requiring in principle only that the domain be a metric space.

The classic paper on nearest neighbor pattern classification is due to
Cover and Hart~\cite{cover-hart}; a textbook treatment appears in
Duda et al.~\cite{pattern-classification}. Both presentations adopt a
probabilistic setting, demonstrating that if the samples are independent
and identically-distributed (iid),
the probability of error converges to no more than twice
the optimal probability of error, the so-called Bayes risk.
In a fully deterministic setting, since the Bayes risk is zero, this amounts
to showing that the nearest neighbor rule with iid sampling converges to the
true pattern. Cover~\cite{cover-estimation} extends these results to the
estimation problem.

Obviously iid sampling is almost certain to produce samples that are
superfluous in the sense that the prediction remains equally accurate even
if these samples are removed. Superfluous samples are harmful in two senses:
first, sampling may be---and usually is---difficult in one way or another;
second, it is computationally more expensive to search for a point's nearest
neighbor as the size of the sample set increases.

The latter concern can be addressed by \emph{editing} techniques, in which a
large set of preclassified samples is shorn down by deleting samples
according to some rule. Wilson~\cite{wilson-edited} shows that convergence
holds when one deletes samples that are misclassified by their
$m \geq 2$ nearest neighbors, and then uses the nearest neighbor rule on the
remaining samples. (Wagner~\cite{wagner-edited} simplifies the proof
considerably.)
An conceptually simpler algorithm involving Voronoi diagrams is given in
\cite{pattern-classification}, but this is unlikely to be practical for
reasons described in \secref{sec:heu:nmc:Knn}.

Of course editing can only delete samples already taken and cannot address
the desire to sample parsimoniously in the first place.
This is achieved by \emph{selective sampling}, wherein each sample
is selected from a pool of candidates according to some heuristic function.
The trick lies in identifying some heuristic such that the odds of choosing a
superfluous (or otherwise ``low-information'') sample are reduced.
Selective sampling falls within the broader paradigm of active learning,
which is surveyed by Settles~\cite{settles}.\footnote{Settles uses the
  term \emph{pool-based sampling}, but selective sampling seems to be
  the term of art among researchers using nearest neighbor techniques.}

Fujii et al.~\cite{fujii} present a nearest neighbor algorithm for
word sense disambiguation using selective sampling. The problem is interesting
because the domain space is non-Euclidean, but the selection heuristic is
quite specific to problems in natural language.
Lindenbaum et al.~\cite{lindenbaum} give a more general treatment,
including very abstract descriptions of the selection heuristic.
However, they assume that the domain is Euclidean and that the true pattern
conforms to a particular random field model.
The selection heuristic is also complex and computationally expensive.

This paper seeks to take the intuition of selective sampling back to the
extremely general setting of \cite{cover-hart}, assuming not much more than
a metric domain on which exists a probability measure. We will give three
selection heuristics and prove that their nearest neighbor rule predictions
converge to the true pattern; furthermore, the first and third of these
algorithms are computationally cheap, with complexity growing only linearly
in the number of samples in the naive implementation.
We believe that proving convergence in such a general setting with such
simple algorithms constitutes an important advance in the art.

Following the present introductory section,
in \secref{sec:prelim} we establish the problem's formal setting.
\secref{sec:heu} contains the key convergence proofs,
plus additional results and remarks relating to the practical use of the
methods. We conclude the paper and indicate avenues for future research,
including the crucial question of convergence rates, in \secref{sec:conc}.

\section{Preliminaries}
\label{sec:prelim}

In this section we lay down some common definitions and notation that we will
use throughout the rest of the paper. The object of our efforts will be to
approximate a classifier function $f : X \to Y$,
the so-called \emph{true function}, where the domain $X$ is a metric
space equipped with metric $d$ and probability measure $\mu$; and the
codomain $Y$ is any countable set. We approximate $f$ using a sequence
of \emph{samples} $\{z_n\}_{n=1}^{\infty}$ from $X$ collected into sets
$Z_n = \{ z_1 , z_2 , \ldots , z_n \}$. The \emph{prediction function}
$\zeta_n : X \to Y$ operates via the nearest neighbor rule on $Z_n$; that is:
\begin{equation}
  \label{eq:NNR}
  \zeta_n(x) = f(z_\iota) \textrm{ where }
  \iota = \arg \min_{i \leq n} d(x,z_i)
  \textrm{.}
\end{equation}
(If more than one sample achieves the minimum distance to $x$, choose one
uniformly at random.)

It should be noted that in contrast to the works cited in \secref{sec:intro},
the true function is fully determined, not probabilistic.
We have chosen a deterministic setting for two reasons.
In the first place, we confess, it makes the analysis much easier.
More fundamentally, however, the algorithms we develop in \secref{sec:heu:nmc}
break with the concept that the nearest sample to any point approaches
arbitrarily near to the point, which is a critical assumption
underlying the calculation of prediction error in terms of Bayes risk.
Recovering the probabilistic setting is an area of future research and
discussed in \secref{sec:conc}.

Given a point $x \in X$, the \emph{(open) $\epsilon$-ball} about $x$,
denoted $B_\epsilon(x)$, is the set of points at distance strictly less than
$\epsilon$ from $x$. The \emph{support} of $\mu$, or $\supp(\mu)$, are the
points of $x \in X$ such that the $\epsilon$-ball about $x$ has positive
measure for all $\epsilon > 0$.

We will be concerned with two types of convergence.
Let $\{f_n\}_{n=1}^{\infty}$ be a sequence of functions.
Given $S \subseteq X$, the sequence
\emph{converges pointwise to $f$ on $S$} iff
\begin{equation*}
  \forall s \in S : \lim_{n \to \infty} f_n(s) = f(s)
  \textrm{.}
\end{equation*}
We also write: $f_n \to f$ pointwise on $S$;
or for any particular $s \in S$: $f_n(s) \to f(s)$.
Furthermore, the sequence of functions \emph{converges in measure to $f$} iff
it converges pointwise to $f$ on all of $X$ save for a set of measure zero.
We write this more briefly as: $f_n \to f$ in measure.

When we say an event occurs \emph{almost surely}, \emph{almost certainly},
or suchlike, we mean that it occurs with probability one with respect
to the probability measure $\mu$;
the phrases \emph{almost never}, \emph{almost impossible}, and so forth,
are attached to an event with probability zero.
In this work, these terms will typically be associated with invocations and
implications of the following classical lemma in probability theory.

\begin{bclem2}
  Let $\{E_n\}_{n=1}^\infty$ be a sequence of independent events in some
  probability space.
  If
  \begin{equation*}
    \sum_{n=1}^\infty \Pr[E_n] = \infty
  \end{equation*}
  then
  \begin{equation*}
    \Pr[\limsup E_n] = 1
    \textrm{.}
  \end{equation*}
  In other words, almost surely $E_n$ occurs for infinitely many $n$.
\end{bclem2}

Our goal is to show that $\zeta_n \to f$ pointwise on $\supp(\mu)$
almost surely under certain instantiations of the following stochastic process:
\begin{enumerate}
\item $n \gets 1$.
\item\label{oegloop} Select at random (per $\mu$) a $\kappa(n)$-set
  $S \subseteq X$ of \emph{candidates}.
\item $z_n \gets \arg \max_{s \in S} \Phi(s,Z_{n-1})$.
  (If there is more than one candidate that achieves the maximum,
  choose one uniformly at random.)
\item $n \gets n+1$.
\item Go to Step~\ref{oegloop}.
\end{enumerate}
We call this process $\mathcal{S}(\kappa,\Phi)$:
it is parameterized by a function $\kappa : \mathbb{Z}^+ \to \mathbb{Z}^+$
that determines the number of candidates;
and a \emph{selection heuristic}
$\Phi : X \times \mathcal{P}(X) \to \mathbb{R}$ that somehow expresses how
``desirable'' a candidate would be as a sample,
given the preceding samples.\footnote{In \cite{fujii,lindenbaum} the selection
  heuristic is called the \emph{utility function}, but we prefer the more
  elastic term.}

Finally, if $X$ is \emph{separable}---that is, if it contains a countable dense
subset---then we will show that in fact $\zeta_n \to f$ in measure as an
immediate corollary of the corresponding pointwise convergence result.

\subsection{Selective sampling may fail badly}
\label{sec:prelim:fail}

Before starting in earnest, it may be valuable to understand how an
arbitrary sequence of samples may fail to converge to $f$. Indeed, we shall
give an example such that the prediction functions $\zeta_n$ have monotonically
decreasing accuracy. Although our example will be rather contrived, it will
show that some care must be taken when designing the sampling process.

Let $X = [0,1] \subset \mathbb{R}$ with Euclidean metric $d$ and Lebesgue
measure $\mu$; and let $Y = \{0,1\}$. Now let
\begin{equation*}
  X_1 = \{0\} \cup \bigcup_{i=1}^{\infty}
  \left( \left. \frac{1}{2^i} , \frac{3}{2^{i+1}} + \epsilon_i \right] \right.
  \textrm{ where }
  0 < \epsilon_i < \frac{1}{2^{i+1}}
\end{equation*}
and take $f$ to be the indicator function on $X_1$.

Suppose we take samples according to the process $z_1 = 0$ and $z_n = 2^{2-n}$
for $n \geq 2$. Observe that for each $\zeta_n$ where $n \geq 2$, we are
newly accurate on a set of measure less than $\epsilon_{n-1}$,
but newly inaccurate on a set of measure greater than $2^{-n}$; and since
$\epsilon_{n-1} < 2^{-n}$ by definition, we become less accurate with each
increment of $n$.

It should also be observed that the sample sequence corresponds to an
intuitively very reasonable selection method: $z_n$ is the midpoint of two
samples with different $f$-values (for all $n \geq 3$).

\section{Selective sampling heuristics}
\label{sec:heu}

We consider two choices for the selection heuristic $\Phi$:
distance to sample set and (two variants of) non-modal count. Of the two,
non-modal count is certainly the more interesting and useful; the distance to
sample set heuristic mostly mimics the action of iid sampling, although it
turns out to be a useful way to present some basic ideas that will be used
again in more complicated settings.

Let us articulate some definitions that connect $f$ with the geometry and
measure of $X$.
For any $y \in Y$, let $X_y = f^{-1}(y) = \{ x \in X; f(x)=y \}$.
A subset $R \subseteq X$ is called \emph{$f$-connected} iff
$R$ is connected and there exists $y \in Y$ such that $R \subseteq X_y$.
An $f$-connected set that is maximal by inclusion in $X$ is an
\emph{$f$-connected component}. Consider a point $b \in X$ on the boundary of
an $f$-connected component: we say $b$ is an \emph{$f$-boundary point} iff
$f(b) = y$ yet $\mu(B_\epsilon(b) \cap (X-X_y)) > 0$ for all $\epsilon > 0$.
The \emph{$f$-boundary} is simply the set of all $f$-boundary points.

As a general rule, we only consider points off the $f$-boundary.
This assumption is necessary for the nearest neighbor rule to be valid, since
any $f$-boundary point can have a ``misleading'' sample arbitrarily close to
it. In order to free ourselves of the burden of the $f$-boundary, we typically
assume that it has measure zero---which is generally a very natural
assumption, provided that $Y$ is countable.

\subsection{Distance to sample set}
\label{sec:heu:dist}

The \emph{distance to sample set heuristic} $\Phi$ is defined very simply as:
\begin{equation}
  \label{eq:dist}
  \Phi(x,S) = d(x,S) = \inf \{ d(x,s); s \in S \}
  \textrm{.}
\end{equation}
This leads to our first convergence theorem.

\begin{thm}
  \label{thm:dist-convergence-support}
  Let $\{z_n\}_{n=1}^{\infty}$ be determined by the process
  $\mathcal{S}(\kappa,\Phi)$ where
  $\kappa : \mathbb{Z}^+ \to \mathbb{Z}^+$ is such that,
  for any $0 < \rho \leq 1$,
  \begin{equation*}
    \sum_{n=1}^\infty \rho^{\kappa(n)} = \infty
  \end{equation*}
  and $\Phi$ is defined by \eqref{eq:dist}.
  If $x \in \supp(\mu)$ is not an $f$-boundary point,
  then $\zeta_n(x) \to f(x)$ with probability one.
\end{thm}

\begin{proof}
  Let $Q_n^\epsilon \subseteq X$ denote the points $q$ such that
  $\Phi(q,Z_n) \geq \epsilon$, and let $Q^\epsilon = \limsup Q_n^\epsilon$
  (ie, $q \in Q^\epsilon$ iff $q$ appears in infinitely many $Q_n^\epsilon$).
  For any point $x \in \supp(\mu)$ that is not an $f$-boundary point,
  there exists $r > 0$ such that $\mu(B_r(x)-X_{f(x)})=0$; we fix some such
  point $x$ and distance $r$.

  Denote by $E_n$ the event that one of the $\kappa(n+1)$ candidates
  for $z_{n+1}$ lies in $B_{r/2}(x)$, while all the others lie in
  $X - Q_n^{r/2}$. Hence
  \begin{equation*}
    \Pr[E_n]
    =
    \kappa(n+1) \mu(B_{r/2}(x)) (1 - \mu(Q_n^{r/2}))^{\kappa(n+1)-1}
    \textrm{.}
  \end{equation*}
  Observe that $\mu(B_{r/2}(x))$ is positive (since $x \in \supp(\mu)$) and
  constant with respect to $n$.
  $\mu(Q_n^{r/2})$ does vary with $n$, of course, but
  $Q_{n+1}^{r/2} \subseteq Q_n^{r/2}$, so the measure is nonincreasing in $n$.
  Therefore the quantity $(1 - \mu(Q_n^{r/2}))$ can be bounded below by some
  $0 < \rho \leq 1$ for large $n$
  (almost always).\footnote{The case that $\mu(Q_n^{r/2})=1$ for large $n$ occurs
  with probability zero since it requires that
  all relevant $Z_n$ lie entirely outside $\supp(\mu)$,
  which can be excluded by the First Borel-Cantelli Lemma.}

  Suppose $\kappa(n) \leq k$ for infinitely many $n$; then for these $n$,
  \begin{equation*}
    \Pr[E_n] \geq \mu(B_{r/2}(q)) \rho^{k-1}
    \textrm{.}
  \end{equation*}
  Since the right-hand side is a constant greater than zero,
  the sum of the $\Pr[E_n]$ must diverge to infinity.
  On the other hand, if the sequence of $\kappa(n)$ has no bounded subsequence,
  then for sufficiently large $n$, $\kappa(n) \geq \mu(B_{r/2}(q))^{-1}$.
  Thus
  \begin{equation*}
    \Pr[E_n]
    \geq
    (1 - \mu(Q_n^{r/2}))^{\kappa(n+1)-1}
    \geq
    \rho^{\kappa(n+1)}
  \end{equation*}
  for large $n$; it follows that $\sum_{n=1}^\infty \Pr[E_n] = \infty$.
  So in either case the Second Borel-Cantelli Lemma tells us that
  $E_n$ occurs infinitely often with probability one.

  In the event $E_n$, let $q$ denote the candidate in $B_{r/2}(x)$.
  If $q \notin Q_n^{r/2}$, then there exists $i \leq n$ such that
  $z_i \in B_{r/2}(q)$---but since $B_{r/2}(q) \subseteq B_r(x)$,
  this contradicts $x \in Q^r$.
  Hence $q \in Q_n^{r/2}$, so it will be chosen as $z_{n+1}$,
  since all other candidates will have smaller $\Phi$-values.
  Consequently, $x \notin Q_{n+i}^r$ for any $i \geq 1$---which also
  contradicts $x \in Q^r$.

  Thus $x \notin Q^r$, meaning that a sample is eventually placed less than
  distance $r$ from $x$. Such a sample would have the same $f$-value as $x$
  with probability one, and we conclude that
  $\zeta_n(x) \to f(x)$ almost always.
\end{proof}

\begin{rem}
  \label{rem:borel-cantelli}
  Given any $x \in \supp(\mu)$,
  one can use the following ``brute force'' argument
  to prove that a sample will eventually be placed arbitrarily close to
  $x$: for any $\epsilon > 0$, $\mu(B_\epsilon(x))>0$, so \emph{all}
  of the $\kappa(n+1)$ candidates for $z_{n+1}$ will lie in
  $B_\epsilon(x)$ infinitely often because
  \begin{equation*}
    \sum_{n=1}^\infty \mu(B_\epsilon(x))^{\kappa(n+1)} = \infty
    \textrm{.}
  \end{equation*}
  This argument,
  while a formally correct use of the Second Borel-Cantelli Lemma,
  is sterile in substance since
  it does not suggest why selective sampling has any advantage over
  iid sampling; indeed, it suggests that iid sampling (ie, $\kappa(n)=1$)
  is superior. Our proof, on the other hand, demonstrates that it suffices for
  \emph{one} candidate to appear near $x$, provided that all others lie in
  regions with lower $\Phi$-values.

  A key theme of this paper will be avoiding ``brute force'' arguments;
  or, when they seem to be unavoidable, demonstrating how to mitigate their
  impact.
\end{rem}

\begin{rem}
  \label{rem:kappa}
  It is not difficult to see that $\kappa$ is admissible if
  $\kappa(n) \leq k$ infinitely often for some constant $k$.
  It is also possible that $\kappa(n) \to \infty$,
  provided that it grows very slowly; consider, for example,
  \begin{equation*}
    \kappa(n) \leq H_{\lceil \lg(n+1) \rceil} =
    \sum_{i=1}^{\lceil \lg(n+1) \rceil} \frac{1}{i}
    \textrm{.}
  \end{equation*}
  This can be proven admissible by means of the Cauchy Condensation Test.
  The practical value in selecting an unbounded $\kappa$ is to help
  ``find'' the subsets of $X$ with high $\Phi$-values, since the measure
  of such subsets decreases as $n \to \infty$.
\end{rem}

Cover and Hart establish the following lemma in \cite{cover-hart}.

\begin{lem}[Cover and Hart~\cite{cover-hart}]
  \label{lem:cover-hart}
  Suppose $X$ is a separable metric space, and let $S \subseteq X$.
  If $S \cap \supp(\mu) = \emptyset$, then $\mu(S)=0$.
\end{lem}

With \lemref{lem:cover-hart}, convergence in measure flows immediately from
\thmref{thm:dist-convergence-support}.

\begin{cor}
  Suppose $\mathcal{S}(\kappa,\Phi)$ is as described by
  \thmref{thm:dist-convergence-support}.
  If $X$ is separable, and the $f$-boundary has measure zero, then
  $\zeta_n \to f$ in measure with probability one.
\end{cor}

\begin{proof}
  If $x \in \supp(\mu)$ is not an $f$-boundary point, then
  $\zeta_n(x) \to f(x)$ almost surely by \thmref{thm:dist-convergence-support}.
  Hence any point that does not converge to its $f$-value
  is either an $f$-boundary point or outside $\supp(\mu)$---but
  since the $f$-boundary has measure zero by assumption,
  and $\mu(X - \supp(\mu))=0$ by \lemref{lem:cover-hart},
  the set of all such points has measure zero.
\end{proof}

\subsection{Non-modal count}
\label{sec:heu:nmc}

The distance to sample set heuristic produces samples that tend to
``spread out'' over $X$ geometrically
(unlike iid sampling, which is sensitive only to $\mu$).
If one's goal is to approximate the true function $f$, however, this is still
not the most efficient arrangement of samples: what one really wants to do is
place ``very many'' samples near the $f$-boundary and ``very few'' samples
elsewhere.

Predictions under the nearest neighbor rule are naturally visualized as
\emph{Voronoi diagrams} (also known as \emph{Voronoi tessellations},
\emph{Voronoi decompositions}, or \emph{Dirichlet tessellations}) that
partition $X$ according to its nearest neighbor in $Z_n$;
see, for instance, \S4 of Devadoss and O'Rourke~\cite{devadoss-orourke}.
We motivate the point in the preceding paragraph by examining Voronoi
diagram predictions of the ``bat insignia'' \cite{bat-insignia} under
different sampling methods. For this example, $X$ is a subset of the
Euclidean plane; $Y = \{0,1\}$; and $\mu$ is the uniform distribution.

\begin{figure}
  \centering
  \subfloat[$N=1000$]{
    \includegraphics[scale=0.5]{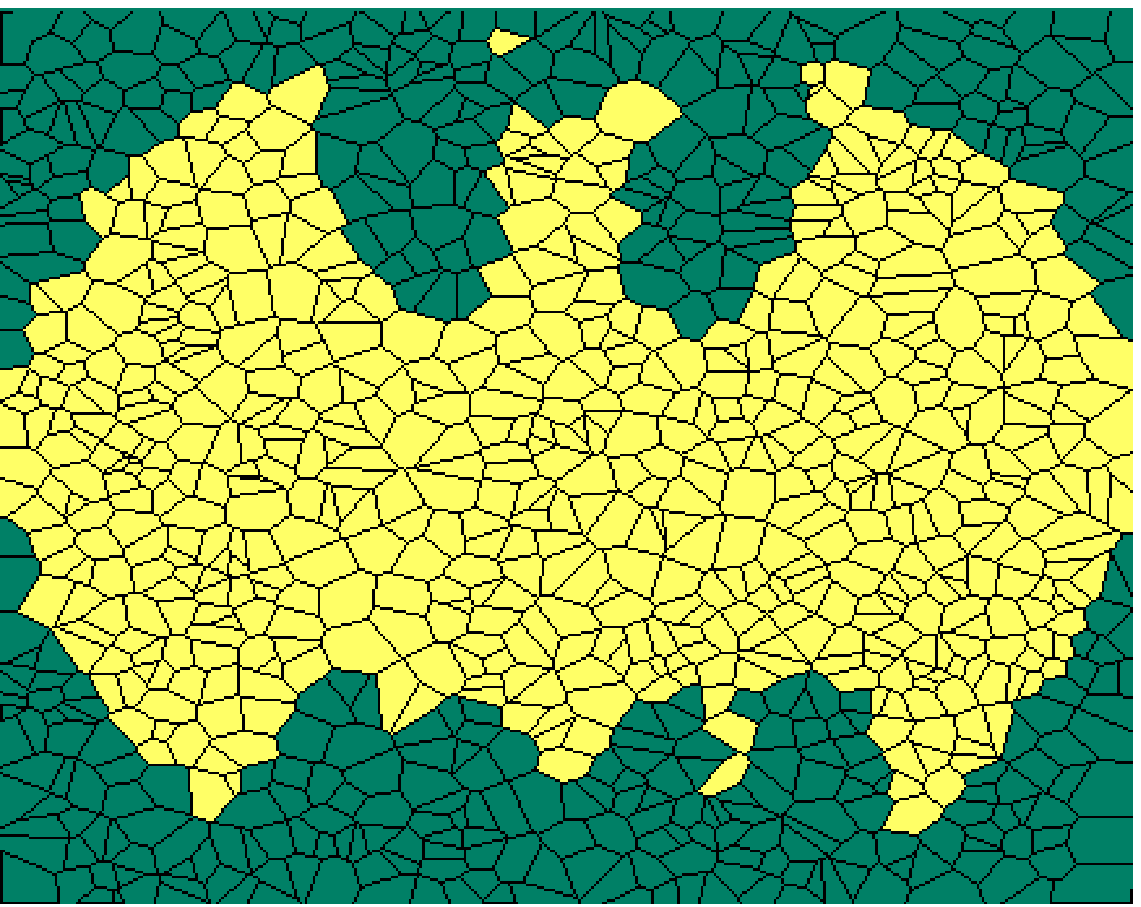}
  }
  \qquad
  \subfloat[$N=5000$]{
    \includegraphics[scale=0.5]{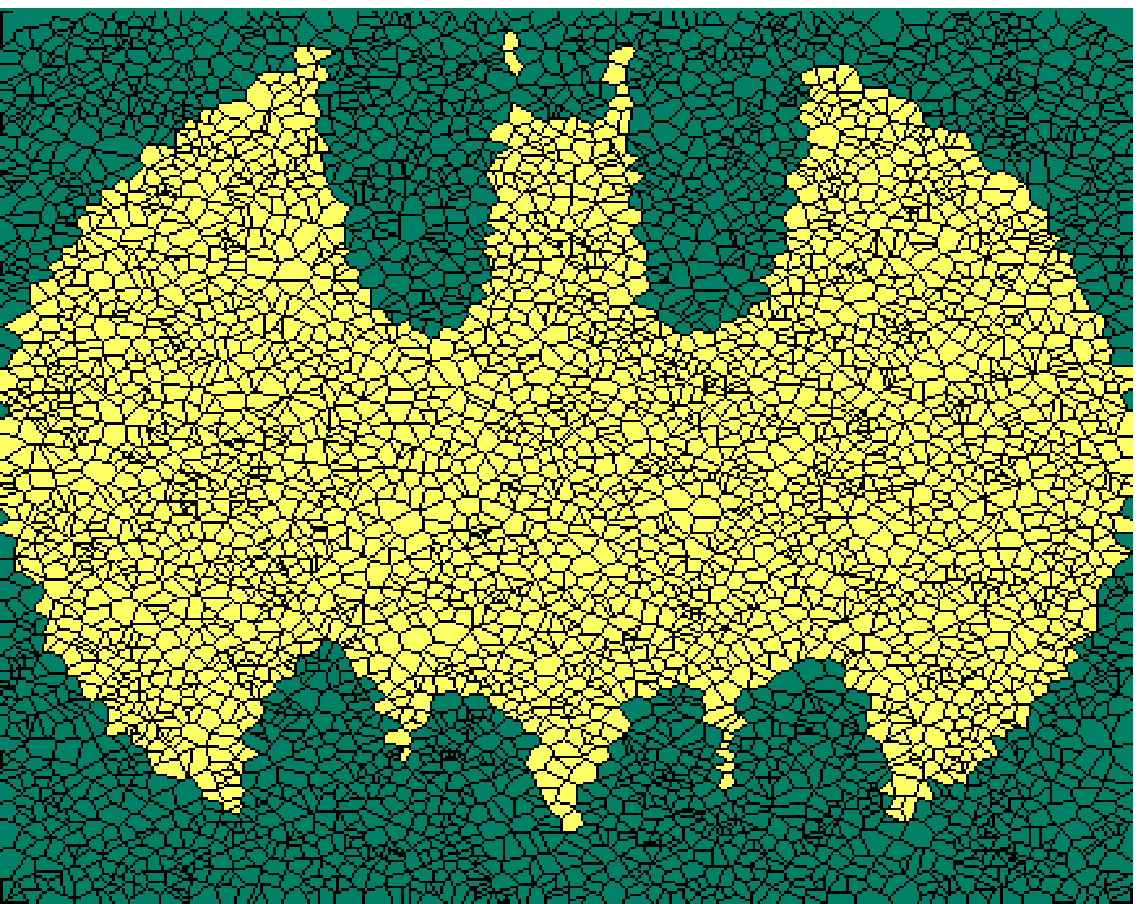}
  }
  \caption{Voronoi diagram predictions of the bat insignia~\cite{bat-insignia} with $N$ iid samples}
  \label{fig:vor-rnd}
\end{figure}

\begin{figure}
  \centering
  \subfloat[$N=1000$]{
    \includegraphics[scale=0.5]{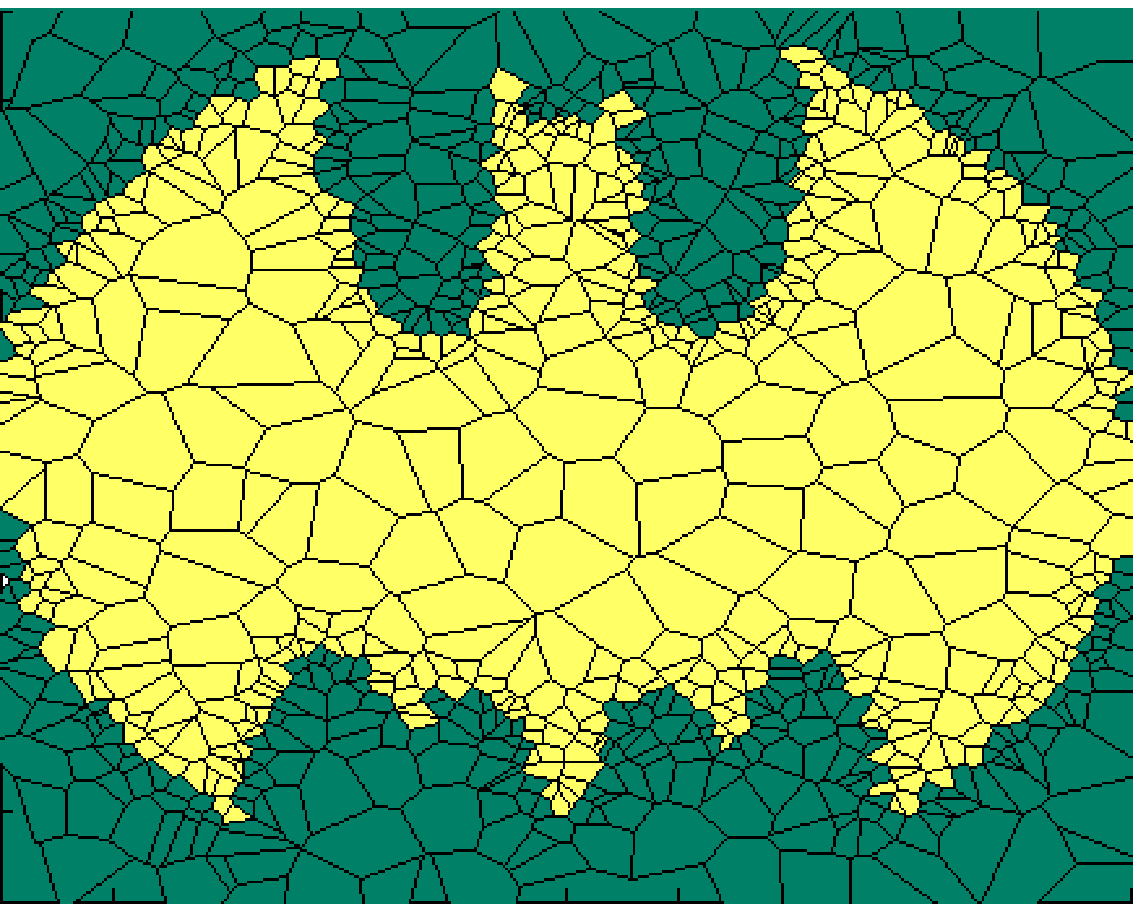}
  }
  \qquad
  \subfloat[$N=5000$]{
    \includegraphics[scale=0.5]{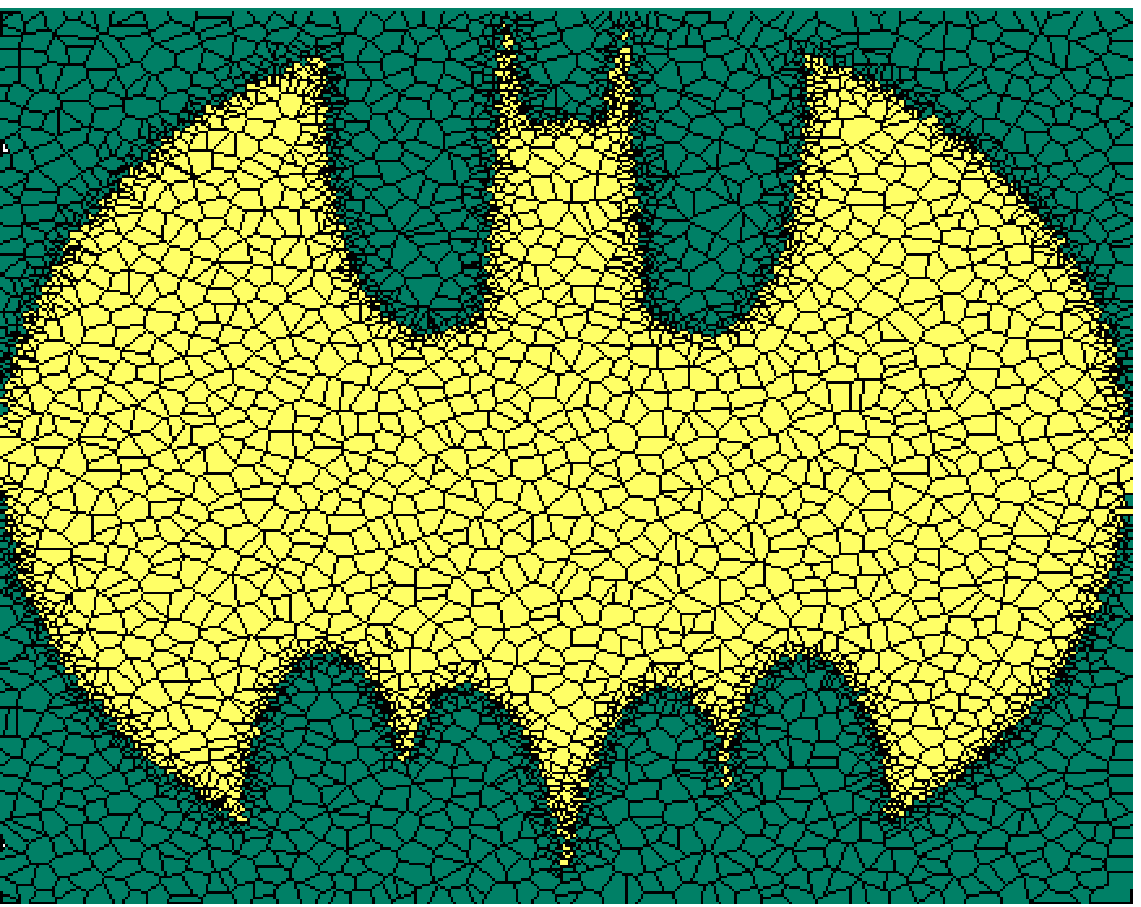}
  }
  \caption{Voronoi diagram predictions of the bat insignia~\cite{bat-insignia} with 20 initial iid samples followed by $N$ samples chosen according to $\mathcal{S}(\kappa,\Phi)$ with $\kappa(n)=10$ and $\Phi$ giving the non-modal count per the 6-nearest neighbors (see \secref{sec:heu:nmc:Knn})}
  \label{fig:vor-nmc}
\end{figure}

In \figref{fig:vor-rnd}, the diagrams are constructed over sets of
iid samples;
note that the cells are roughly the same size in each diagram,
exactly as one would expect with iid samples.
\figref{fig:vor-nmc}, on the other hand, shows diagrams constructed with
a selective sampling method;
here we observe that the cells are very small near the $f$-boundary and
large elsewhere. Comparing each diagram in \figref{fig:vor-rnd} with its
correspondent in \figref{fig:vor-nmc}, clearly the latter achieves a more
accurate prediction with (virtually) the same number of samples.

The \emph{non-modal count heuristic} $\Phi$ is defined as:
\begin{equation}
  \label{eq:nmc}
  \Phi(x,S) = |V_S(x)| - \modefreq_f(V_S(x))
\end{equation}
where $V_S : X \to \mathcal{P}(S)$ maps $x$ to some (possibly empty) set
of neighbors in $S$; and
$\modefreq_f(A)$ denotes the frequency of the mode of $A$ under $f$.
Below we give two alternatives for $V_S$ and prove convergence for both.

\subsubsection{Voronoi neighbors}
\label{sec:heu:nmc:vor}

We have already observed that Voronoi diagrams are the natural visualization
of the nearest neighbor rule, so it is natural to expect that their formal
properties might be useful. As it turns out, we will not require very much
from the theory of Voronoi diagrams, save the following definition:
given $x \in X$ and $S \subseteq X$, we say that $v \in S$ is a
\emph{Voronoi neighbor} of $x$ iff there exists a point $c \in X$ for which
$d(x,c) < d(v,c) \leq d(s,c)$ for any $s \in S$.
With reference to \eqref{eq:nmc} above, then, we can define $V_S(x)$ as the
set of all Voronoi neighbors of $x$ with respect to $S$.

Stated another way, $v$ is a Voronoi neighbor of $x$ if
there is some point in $X$ (ie, $c$) whose nearest neighbor in $S$ is $v$,
yet whose nearest neighbor in $S \cup \{x\}$ is $x$.
(Observe that $V_S(x) = \emptyset$ iff $x \in S$.
Furthermore, if $x \notin S$, then
its nearest neighbor in $S$ is always a Voronoi neighbor: let $c=x$.)
This definition is well-suited to analyze the evolution of predictions:
if a candidate with positive non-modal count is selected as a sample, the
prediction function changes, since at least one point will be predicted
differently.

The following lemma, which shows that Voronoi neighbors are preserved in
sufficiently small neighborhoods, will be helpful in our investigations.

\begin{lem}
  \label{lem:vor}
  Let $S \subseteq X$.
  For any $x \in X - S$, if $v \in V_S(x)$,
  then there exists $\epsilon > 0$ such that
  for all $x' \in B_\epsilon(x)$, $v \in V_S(x')$.
  Furthermore, if $v$ is the nearest neighbor of $x$ in $S$,
  then the previous statement holds for all $0 < \epsilon \leq d(x,v)$.
\end{lem}

\begin{proof}
  By definition there exists $c \in X$ such that
  $d(x,c) < d(v,c) \leq d(s,c)$ for any $s \in S$.
  Let $0 < \epsilon \leq d(v,c) - d(x,c)$.
  By the triangle inequality, for any $x' \in B_\epsilon(x)$ we have
  \begin{equation*}
    d(x',c) \leq d(x',x) + d(x,c) < \epsilon + d(x,c) \leq d(v,c)
  \end{equation*}
  so $v \in V_S(x')$. The final claim is established by letting $c=x$.
\end{proof}

Before proceeding to our next convergence result, we will require two more
definitions. An $f$-connected set $R$ is said to be \emph{$f$-contiguous} iff
$R$ is a subset of $\supp(\mu)$ and does not contain $f$-boundary points.
A maximal $f$-contiguous set is an \emph{$f$-contiguous component}.

\begin{thm}
  \label{thm:nmc-vor-convergence-support}
  Let $\{z_n\}_{n=1}^{\infty}$ be determined by the process
  $\mathcal{S}(\kappa,\Phi)$ where
  $\kappa$ is such that, for any $0 < \rho \leq 1$,
  \begin{equation*}
    \sum_{n=1}^\infty \rho^{\kappa(n)} = \infty
  \end{equation*}
  and $\Phi$ is defined by \eqref{eq:nmc},
  with $V_S(x)$ denoting the Voronoi neighbors of $x$ with respect to $S$.
  If $x \in X$ is contained in an $f$-contiguous component of positive measure,
  then $\zeta_n(x) \to f(x)$ with probability one.
\end{thm}

\begin{proof}
  Let $W_n \subseteq X$ be the set of points $w$ such that
  $\zeta_n(w) \neq f(w)$; and let $W = \limsup W_n$.
  We wish to show that no point contained in an $f$-contiguous component
  of positive measure is also contained in $W$.
  Suppose, for the sake of contradiction, that such a point $x$ exists.
  
  Let $C$ denote the $f$-contiguous component containing $x$. 
  We claim that a sample will eventually be placed in $C$ almost surely.
  By assumption, $\mu(C) > 0$.
  Let $G_n$ be the event that all of the $\kappa(n+1)$ candidates
  for $z_{n+1}$ lie in $C$; then
  \begin{equation*}
    \sum_{n=1}^\infty \Pr[G_n] = \sum_{n=1}^\infty \mu(C)^{\kappa(n+1)} = \infty
  \end{equation*}
  and the Second Borel-Cantelli Lemma tell us that $G_n$ occurs infinitely
  often.
  The claim follows immediately; but see \remref{rem:contiguous-components}.

  Every metric space has a unique (up to isometry) completion, so let
  $\bar{X}$ denote the completion of $X$.
  We can extend $f$ to every point in $\bar{X}$ in the following
  way: for any $\bar{x} \in \bar{X} - X$,
  if there exists $y \in Y$ and $\epsilon > 0$ such that
  $B_\epsilon(\bar{x}) \cap X \subseteq X_y$,
  then $f(\bar{x})=y$;
  otherwise $f(\bar{x})$ may be assigned any value in $Y$.
  We also extend $\mu$ to $\bar{X}$ by letting $\mu(\bar{X}-X)=0$.
  In this way we allow ourselves to work in $X$ or $\bar{X}$ interchangeably.

  Consider $\partial W \subseteq \bar{X}$, the boundary of $W$.
  We say that $\partial W$ \emph{crosses} $C$ at $b$ iff
  $b \in \partial W$ is not an $f$-boundary point and
  $B_\epsilon(b)$ has nonempty intersection with
  both $C \cap W$ and $C - W$ for all $\epsilon > 0$.
  If $\partial W$ does not cross $C$ at any point,
  then $x \in C$ implies $C \subseteq W$,
  which is inconsistent with a sample being placed in $C$;
  thus $\partial W$ almost certainly crosses $C$.

  Let $b$ denote a point at which $\partial W$ crosses $C$;
  since it is not an $f$-boundary point,
  there exists $r > 0$ such that $\mu(B_r(b) - X_{f(b)}) = 0$.
  From this it follows that $d(b,z_n) \geq r$ for all $n$ almost surely.
  Now for any $0 < \alpha < r$, we can find $b_1,b_2 \in B_{\alpha}(b)$
  such that $b_1 \notin W$ and $b_2 \in W$.
  Hence for infinitely many $n$, \lemref{lem:vor} implies that
  any point in $B_{r-\alpha}(b_1)$ has a Voronoi neighbor in $Z_n$ with the
  same $f$-value as $b$;
  and similarly that any point in $B_{r-\alpha}(b_2)$ has a Voronoi neighbor
  in $Z_n$ with $f$-value different from $b$. Taking $\alpha$ sufficiently
  small allows us to choose $0 < \epsilon < r$ such that
  \begin{equation*}
    B_{\epsilon}(b) \subset B_{r-\alpha}(b_1) \cap B_{r-\alpha}(b_2)
    \textrm{.}
  \end{equation*}

  Let $Q_n \subset X$ denote the points $q$ such that
  $\Phi(q,Z_n) > 0$, and let $Q = \limsup Q_n$.
  Because every point in $B_\epsilon(b)$ has (at least) two Voronoi neighbors
  with different $f$-values in infinitely many $Z_n$ (almost surely),
  there is an infinite subsequence $\{n_i\}_{i=1}^{\infty}$ such that
  $B_\epsilon(b) \cap X \subseteq Q_{n_i}$;
  so in fact, $B_\epsilon(b) \cap X \subseteq Q$ (with probability one).

  By the definition of $b$, there exists $c \in B_\epsilon(b) \cap C \cap W$.
  Since $c \in C$ means that $c$ is not an $f$-boundary point,
  there exists $\gamma_1 > 0$ such that $\mu(B_{\gamma_1}(c) - X_{f(c)})=0$;
  furthermore, there exists $\gamma_2 > 0$ such that
  $B_{\gamma_2}(c) \subseteq B_\epsilon(b)$.
  Set $\gamma = \min(\gamma_1,\gamma_2)$, and note that $\gamma$ is
  constant with respect to $n_i$.
  Let $E_i$ denote the event that one of the $\kappa(n_i+1)$ candidates for
  $z_{n_i+1}$ lies in $B_\gamma(c)$ while all others lie in $X-Q_{n_i}$. Thus
  \begin{equation*}
    \Pr[E_i] = \kappa(n_i+1) \mu(B_\gamma(c)) (1 - \mu(Q_{n_i}))^{\kappa(n_i+1)-1}
    \textrm{.}
  \end{equation*}
  Since $c \in C \subseteq \supp(\mu)$, $\mu(B_\gamma(c)) > 0$.

  Now for $i$ sufficiently large,
  $\mu(Q_{n_i})$ is arbitrarily close to $\mu(Q)$.
  If $\mu(Q)<1$, then $1-\mu(Q_{n_i})$ is bounded away from zero;
  we can then employ the same argument as in the proof of
  \thmref{thm:dist-convergence-support} to show that
  $E_i$ occurs infinitely often with probability one.
  In the event $E_i$, since the candidate in $B_\gamma(c)$ is the only
  one with positive non-modal count,
  we will have $z_{n_i+1} \in B_\gamma(c) \subseteq B_{\gamma_1}(c)$,
  which implies that $c \notin W$ almost surely.

  If, on the other hand, $\mu(Q)=1$, we resort to the ``brute force''
  event $F_n$ in which all candidates lie in $B_{\gamma_1}(c)$.
  (However, see \remref{rem:measure-Q}.) Now
  \begin{equation*}
    \sum_{n=1}^\infty \Pr[F_n] =
    \sum_{n=1}^\infty \mu(B_{\gamma_1}(c))^{\kappa(n+1)} = \infty
  \end{equation*}
  so the Second Borel-Cantelli Lemma says that $F_n$ occurs infinitely
  often with probability one.
  But if $F_n$ occurs, $c \notin W$ almost surely.

  We have, at last, obtained a contradiction.
  We conclude that if $x$ is contained in an $f$-contiguous component of
  positive measure,
  then almost surely $x \notin W$;
  in other words, $\zeta_n(x) \to f(x)$ with probability one.
\end{proof}

An important aspect of the proof is that, although we ultimately derive at
contradiction at the point $x$, this is effected by placing points close to
$c$, a point in the same $f$-contiguous component as $x$, but otherwise not
assumed to be close to $x$. This means that relatively sparse sampling in the
``middle'' of $f$-contiguous components suffices, as suggested by
\figref{fig:vor-nmc}.

\begin{cor}
  Suppose $\mathcal{S}(\kappa,\Phi)$ is as described by
  \thmref{thm:nmc-vor-convergence-support}.
  If $X$ is separable;
  the $f$-boundary has measure zero;
  and the union of all $f$-contiguous components with measure zero itself has
  measure zero;
  then $\zeta_n \to f$ in measure with probability one.
\end{cor}

\begin{rem}
  \label{rem:contiguous-components}
  If an $f$-contiguous component does not contain at least one sample,
  it is possible for it to be ``overlooked'' by the non-modal count
  heuristic. A ``brute force'' invocation of the Second Borel-Cantelli
  Lemma is therefore used to show that a sample is eventually placed in
  every $f$-contiguous component of positive measure. In practice, however,
  one would prefer to have samples placed in every such component through
  some \emph{a priori} process. For example, if it is known that every
  $f$-contiguous component (about which one cares) has measure at least
  $p$, placing $\max(20,5/p)$ initial iid samples gets a sample
  in each component with over 99 percent confidence
  (as a consequence of the Central Limit Theorem).
\end{rem}

\begin{rem}
  \label{rem:measure-Q}
  If $Q$, the set of points with non-modal counts that never settle to zero,
  has measure equal (or close) to one,
  the proof has a ``brute force'' aspect that
  we claimed to avoid in \remref{rem:borel-cantelli}.
  Some reflection on the problem will show that $X$ has to be somewhat
  ``weird'' for such $Q$ to occur, because we generally expect the Voronoi
  neighbors close to some point to ``block out'' any potential
  neighbors far from the point.
  \prpref{prp:measure-Q-zero} makes this insight precise by showing that,
  in fact, $\mu(Q)=0$ under certain reasonable conditions.
\end{rem}

Recall that a metric space with metric $d$ is a \emph{length space} iff
for any points $x,y$ in the space, $d_I(x,y) = d(x,y)$, where
$d_I$ is the infimum of the lengths of all paths from $x$ to $y$.

\begin{prp}
  \label{prp:measure-Q-zero}
  Suppose $\mathcal{S}(\kappa,\Phi)$ is as described by
  \thmref{thm:nmc-vor-convergence-support}, and assume that all the
  following obtain:
  \begin{enumerate}
  \item $X$ is separable;
  \item the completion of $X$ is a length space; and
  \item for every $x \in \supp(\mu)$ that is not an $f$-boundary point,
    there exists $\omega > 0$ such that $B_\omega(x) \subseteq \supp(\mu)$.
  \end{enumerate}
  Let $Q_n = \{ x \in X; \Phi(x,Z_n) > 0 \}$.
  If the $f$-boundary has measure zero, then
  $\mu( \limsup Q_n ) = 0$ with probability one.
\end{prp}

\begin{proof}
  Let $Q = \limsup Q_n$.
  Suppose, for the sake of contradiction, that there exists
  $q \in Q \cap \supp(\mu)$ off the $f$-boundary.
  There exists $r > 0$ such that $\mu(B_r(q) - X_{f(q)}) = 0$;
  set $\epsilon = \min(r,\omega)$, where $\omega$ is as defined above.
  Note that $\epsilon$ does not depend on $n$.

  Let $\bar{X}$ denote the completion of $X$;
  we can use $X$ and $\bar{X}$ interchangeably by the same techniques used
  in the proof of \thmref{thm:nmc-vor-convergence-support}.
  Let $S_\gamma(q) \subseteq \bar{X}$
  denote the boundary of the $\gamma$-ball about $q$.
  Consider $S_{\epsilon/2}(q)$:
  this set can obviously be covered by a finite set of $\epsilon/4$-balls
  about some points $x_1,x_2,\ldots,x_m \in S_{\epsilon/2}(q)$.

  Let $E_{i,n}$ be the event that one of the $\kappa(n+1)$ candidates for
  $z_{n+1}$ lies in $B_{\epsilon/4}(x_i)$ while all others lie in the set
  \begin{equation*}
    L_{i,n} = B_{\epsilon/4}(x_i) \cup \{ l \in Q_n;
    \forall x \in B_{\epsilon/4}(x_i) : \Phi(l,Z_n) < \Phi(x,Z_n) \}
    \textrm{.}
  \end{equation*}
  Thus
  \begin{equation*}
    \Pr[E_{i,n}] =
    \kappa(n+1) \mu(B_{\epsilon/4}(x_i)) \mu(L_{i,n})^{\kappa(n+1)-1}
    \geq \mu(B_{\epsilon/4}(x_i))^{\kappa(n+1)}
    \textrm{.}
  \end{equation*}
  Because $x_i$ is a limit point,
  there exists $x_i' \in X$ arbitrarily close to $x_i$;
  thus some sufficiently small ball about $x_i'$ will be contained
  in $B_{\epsilon/4}(x_i)$.
  Furthermore, since
  \begin{equation*}
    x_i' \in B_\epsilon(q) \subseteq B_\omega(q) \subseteq \supp(\mu)
  \end{equation*}
  it must be that $\mu(B_{\epsilon/4}(x_i)) > 0$ and
  $\Pr[E_{i,n}]$ is bounded away from zero.
  So the sum of the $\Pr[E_{i,n}]$ diverges to infinity, and
  by the Second Borel-Cantelli Lemma,
  $E_{i,n}$ occurs infinitely often with probability one.
  Now in the event $E_{i,n}$, a candidate in $B_{\epsilon/4}(x_i)$
  will become $z_{n+1}$.
  We conclude that for $n$ sufficiently large, almost certainly
  $Z_n$ has at least one element in every $B_{\epsilon/4}(x_i)$
  where $1 \leq i \leq m$.
  
  Since $q \in Q$, there exist arbitrarily large $n$ such that
  $q \in Q_n$. For such $n$, $q$ has a Voronoi neighbor
  $z_a \in Z_n$ such that $f(z_a) \neq f(q)$; so almost surely
  $d(z_a,q) \geq r$, given that $q$ is not an $f$-boundary point.
  Let $c$ be a point that certifies that $z_a \in V_{Z_n}(q)$.
  
  Suppose $c \notin B_{\epsilon/2}(q)$.
  Now for any $\delta > 0$, there exists $c' \in S_{\epsilon/2}(q)$ with
  \begin{equation*}
    d(q,c') + d(c',c) < d(q,c) + \delta
    \textrm{.}
  \end{equation*}
  This is a consequence of the fact that $\bar{X}$ is a length space:
  one can think of $c'$ as the point where an almost-shortest path
  from $q$ to $c$ intersects $S_{\epsilon/2}(q)$.
  But for large $n$,
  as there is a sample in every $B_{\epsilon/4}(x_i)$,
  there will almost surely be $z_b \in Z_n$ such that
  $d(z_b,c') < \epsilon/2 = d(q,c')$. So:
  \begin{equation*}
    d(z_b,c) \leq d(z_b,c') + d(c',c) < d(q,c') + d(c',c) < d(q,c) + \delta
    \textrm{.}
  \end{equation*}
  Yet since $\delta$ may be arbitrarily small, in fact $d(z_b,c) < d(q,c)$.
  This contradicts the character of $c$.

  Alternately, suppose $c \in B_{\epsilon/2}(q)$; we take $c'$ as before,
  except this time on a path from $z_a$ to $c$.
  With $z_b$ as before, $d(z_b,c') < \epsilon/2 \leq r/2 \leq d(z_a,c')$.
  Reasoning exactly as above,
  $d(z_b,c) < d(z_a,c)$---which is also a contradiction.
  
  Both alternatives being contradictory, we conclude that $q \notin Q$;
  so in fact, $Q$ is disjoint from $\supp(\mu)$, save for the
  $f$-boundary.
  As $X$ is separable, \lemref{lem:cover-hart} applies,
  and given that the $f$-boundary has measure zero,
  we conclude that $\mu(Q)=0$.
\end{proof}

Observe that \prpref{prp:measure-Q-zero} is satisfied for
a Euclidean space; or indeed for the space of rational points of any
dimension, since its completion is Euclidean.

\subsubsection{$K$-nearest neighbors}
\label{sec:heu:nmc:Knn}

Reasoning in terms of Voronoi neighbors is desirable since it closely
reflects the evolution of the Voronoi diagrams $\zeta_n$;
however, calculation of the Voronoi neighbors is typically difficult.
The definition given above is unsuitable for computation, since it does not
suggest any method for actually finding the ``certifying'' point $c$.
We are unaware of an algorithm for finding Voronoi neighbors that does not
involve, in one way or another, constructing the Voronoi diagram itself.
Although there are a number of algorithms for constructing these
diagrams---some are presented in \S4 of
Devadoss and O'Rourke~\cite{devadoss-orourke}---they are
unlikely to be practical for our purposes,
except perhaps when $X$ is a Euclidean plane.

A reasonable alternative is to use metric closeness
as a kind of approximation for adjacency in the Voronoi sense.
Fix some integer $K \geq 2$; given $S \subseteq X$,
let $\mathcal{V}_x$ denote the family of $K$-sets
$V \subseteq S$ that minimize distance to $x$.
With reference to \eqref{eq:nmc}, then, let
\begin{equation*}
  V_S(x) =
  \left\{
  \begin{array}{ll}
    \emptyset, & \textrm{if $x \in S$}\\
    \arg \min_{V \in \mathcal{V}_x} \modefreq_f(V), & \textrm{otherwise}
  \end{array}
  \right.
  \textrm{.}
\end{equation*}
We say that $V_S(x)$ is the set of \emph{$K$-nearest neighbors} of $x$
with respect to $S$.\footnote{Note that defining $V_S(x)$ in this way causes
  $\Phi(x,S)=0$ when $x \in S$, and otherwise returns the neighbor set that
  maximizes $\Phi(x,S)$.}

\begin{lem}
  \label{lem:Knn}
  Let $S \subseteq X$ be countable; for any $x \in X-S$, write
  $S = \{ s_1 , s_2 , \ldots \}$ such that
  \begin{equation*}
    d(x,s_1) \leq d(x,s_2) \leq \cdots
    \textrm{.}
  \end{equation*}
  If there exists $K \geq 1$ such that $d(x,s_K) < d(x,s_{K+1})$, then
  for any
  \begin{equation*}
    0 < \epsilon \leq \frac{1}{2}[d(x,s_{K+1}) - d(x,s_K)]
  \end{equation*}
  it holds that for all $x' \in B_\epsilon(x)$,
  $\{s_1,s_2,\ldots,s_K\}$ are the $K$-nearest neighbors of $x'$
  with respect to $S$.
\end{lem}

\begin{proof}
  For any $x' \in B_\epsilon(x)$ and $1 \leq i \leq K < j$,
  we have the following:
  \begin{align*}
    d(x',s_i)
    &\leq d(x',x) + d(x,s_i)\\
    &< \epsilon + d(x,s_i)
    \leq \epsilon + d(x,s_K)\\
    &\leq \frac{1}{2}[ d(x,s_{K+1}) - d(x,s_K) ] + d(x,s_K)\\
    &= \frac{1}{2}[ d(x,s_{K+1}) + d(x,s_K) ]
    = d(x,s_{K+1}) - \frac{1}{2}[ d(x,s_{K+1}) - d(x,s_K) ]\\
    &\leq d(x,s_{K+1}) - \epsilon
    \leq d(x,s_j) - \epsilon\\
    &< d(x',s_j)
    \textrm{.}
  \end{align*}
  The lemma follows immediately.
\end{proof}

\begin{thm}
  \label{thm:nmc-Knn-convergence-support}
  Let $\{z_n\}_{n=1}^{\infty}$ be determined by the process
  $\mathcal{S}(\kappa,\Phi)$ where
   $\kappa$ is such that, for any $0 < \rho \leq 1$,
  \begin{equation*}
    \sum_{n=1}^\infty \rho^{\kappa(n)} = \infty
  \end{equation*}
  and $\Phi$ is defined by \eqref{eq:nmc},
  with $V_S(x)$ denoting $K$-nearest neighbors of $x$ with respect to $S$
  ($K \geq 2$).
  If $x \in X$ is contained in an $f$-contiguous component of positive
  measure,
  then $\zeta_n(x) \to f(x)$ with probability one.
\end{thm}

\begin{proof}
  We proceed exactly as in the proof of
  \thmref{thm:nmc-vor-convergence-support}
  up to the identification of the point $b$ at which the boundary of
  $W$ crosses $C$.
  For any $Z_n$, let $b_n(i) \in Z_n$ denote the $i$th-nearest neighbor of $b$;
  that is,
  \begin{equation*}
    d(b,b_n(1)) \leq d(b,b_n(2)) \leq \cdots \leq d(b,b_n(n))
    \textrm{.}
  \end{equation*}
  Now consider
  \begin{equation*}
    \epsilon_n = \frac{1}{2} \max_{2 \leq i \leq K} d(b,b_n(i+1)) - d(b,b_n(i))
    \textrm{.}
  \end{equation*}
  Per \lemref{lem:Knn}, any point in $B_{\epsilon_n}(b)$
  will have at least $b_n(1)$ and $b_n(2)$ among its $K$-nearest neighbors
  with respect to $Z_n$
  (assuming $\epsilon_n > 0$).

  Suppose $\epsilon = \liminf \epsilon_n > 0$.
  As $b$ is on the boundary of points predicted wrongly infinitely often,
  there exist $q_1,q_2 \in B_{\epsilon}(b) \cap X$ such that
  $\zeta_n(q_1) \neq \zeta_n(q_2)$ for infinitely many $n$
  (ie, take $q_1$ on the $W$ side and $q_2$ on the $X-W$ side of
  $\partial W$).
  For these $n$, it must be that $f(b_n(i)) \neq f(b_n(j))$
  for some $1 \leq i < j \leq K$,
  which implies that $B_{\epsilon}(b) \cap X \subseteq Q$.
  We can proceed from here exactly as in the proof of
  \thmref{thm:nmc-vor-convergence-support}:
  find $c \in C \cap B_\epsilon(b)$ such that $c \in W$, etc.
  (However, see \remref{rem:measure-Q-Knn}.)

  If on the other hand $\epsilon = 0$, we must fall back upon
  ``brute force'' to demonstrate that a sample will eventually be placed
  arbitrarily close to any point in $\supp(\mu)$ with probability one.
  (If one considers the definition of $\epsilon_n$, however, the case that
  $\epsilon=0$ seems to defy common sense; \remref{rem:spheres-measure-zero}
  clarifies this intuition.)

  In any event, the theorem is proved.
\end{proof}

\begin{cor}
  Suppose $\mathcal{S}(\kappa,\Phi)$ is as described by
  \thmref{thm:nmc-Knn-convergence-support}.
  If $X$ is separable;
  the $f$-boundary has measure zero;
  and the union of all $f$-contiguous components with measure zero itself has
  measure zero;
  then $\zeta_n \to f$ in measure with probability one.
\end{cor}

\begin{rem}
  \label{rem:choice-K}
  What is a reasonable value for $K$? If we view the $K$-nearest neighbors
  as a kind of approximation of the Voronoi neighbors, then setting $K$
  to the expected number of Voronoi neighbors is reasonable.
  Suppose $X$ is an $m$-dimensional Euclidean space:
  letting $K$ be this expectation, Tanemura~\cite{tanemura03,tanemura05}
  reports the exact results
  \begin{align*}
    m=2 &\implies K = 6\\
    m=3 &\implies K = \frac{48 \pi^2}{35} + 2 = 15.535457\ldots\\
    m=4 &\implies K = \frac{340}{9} = 37.777\ldots
  \end{align*}
  and derives the value $K = 88.56\ldots$ empirically for $m=5$.
  We are not aware of results for $m \geq 6$, although it is reasonable to
  assume from the preceding that $K$ grows exponentially in $m$.
\end{rem}

\begin{rem}
  \label{rem:measure-Q-Knn}
  For the same reasons as noted in \remref{rem:measure-Q}, it is useful
  that $\mu(Q)$ be small.
  Suppose $X$ is separable, and let $x \in \supp(\mu)$ be off the $f$-boundary.
  If $x \in Q$, then there (almost surely) must exist $\gamma > 0$ such that
  $z_n \notin B_\gamma(x)$ for any $n$. But the Second Borel-Cantelli Lemma
  tells us that for infinitely many $n$, all of the $\kappa(n)$ candidates
  for $z_n$ appear in $B_\gamma(x)$, given that $\mu(B_\gamma(x)) > 0$;
  so $z_n \in B_\gamma(x)$, which is a contradiction.
  Thus $x \notin Q$,
  from which it follows by \lemref{lem:cover-hart} that $\mu(Q)=0$,
  provided that the $f$-boundary has measure zero.
  This is the analogue to \prpref{prp:measure-Q-zero} for the
  $K$-nearest neighbors variant of the non-modal count heuristic.
\end{rem}

\begin{rem}
  \label{rem:spheres-measure-zero}
  Let $b$, $b_n(i)$, $\epsilon_n$, and $\epsilon$ be as in the proof of
  \thmref{thm:nmc-Knn-convergence-support}. We have seen that we are forced
  into a ``brute force'' method when $\epsilon = 0$; how can this be mitigated?
  Let $r = \lim_{n \to \infty} d(b,b_n(1))$;
  since $b$ is not an $f$-boundary point, $r > 0$.
  Now consider the shells
  \begin{equation*}
    R_n(b) = \{ x \in X; r \leq d(b,x) \leq d(b,b_n(K)) \}
    \textrm{.}
  \end{equation*}
  If $\epsilon = 0$, it follows that
  \begin{equation*}
    R(b) = \lim_{n \to \infty} R_n(b) = \{ x \in X; d(b,x) = r \}
  \end{equation*}
  and furthermore $\mu(R_n(b)) \to \mu(R(b))$.
  Therefore if $\mu(R(b))=0$, then $\epsilon > 0$ \emph{asymptotically}
  almost always (since otherwise $K$ samples must be placed in a set
  of ever-decreasing measure).
  This condition is satisfied if every sphere---that is, every set of
  points at some fixed distance from a chosen point---has measure zero.
  This is a very natural hypothesis; eg, one intuitively expects that
  every $(n-1)$-dimensional manifold in an $n$-dimensional space has
  measure zero.
\end{rem}

\subsection{The $m$ nearest neighbors rule}
\label{sec:heu:mNNR}

For clarity of exposition we have used the nearest neighbor rule for
prediction, but it is not especially difficult to extend the results to the
$m$ nearest neighbors rule for any fixed integer $m > 1$.
Given any $S \subseteq X$, let $\mathcal{U}_x$ denote the family of
$m$-sets $U \subseteq S$ that minimize distance to $x$. Let
\begin{equation*}
  U_S(x) =
  \left\{ \begin{array}{ll}
    \{ x \}, &\textrm{if $x \in S$}\\
    \arg \min_{U \in \mathcal{U}_x} \modefreq_f(U), &\textrm{otherwise}
  \end{array} \right.
  \textrm{.}
\end{equation*}
Now instead of \eqref{eq:NNR}, we use
\begin{equation}
  \label{eq:mNNR}
  \zeta_n(x) = \mode_f(U_{Z_n}(x))
\end{equation}
for any $x \in X$, with ties broken uniformly at random.\footnote{Note that
  this definition predicts based on the most ambiguous set of neighbors.}

Consider the proof of \thmref{thm:dist-convergence-support}:
given $x \in \supp(\mu)$ off the $f$-boundary, there exists $r > 0$
such that $\mu(B_r(x) - X_{f(x)})=0$, and the proof shows that
some sample will be appear in $B_r(x)$ almost surely.
But in fact the proof, which uses the Second Borel-Cantelli Lemma,
establishes that this almost certainly happens \emph{infinitely often}.
In other words, given any $m > 1$, for sufficiently large $n$,
almost surely $|B_r(x) \cap Z_n| \geq m$, which implies that
$\zeta_n(x) = f(x)$ almost surely.
Hence \thmref{thm:dist-convergence-support} is valid under the $m$ nearest
neighbors rule.

We can also generalize \thmref{thm:nmc-Knn-convergence-support}, although we
must have $K > \lfloor m/2 \rfloor$. Referring back to the proof, let
\begin{equation*}
  \epsilon_n =
  \frac{1}{2} \max_{\lfloor m/2 \rfloor < i \leq K} d(b,b_n(i+1)) - d(b,b_n(i))
  \textrm{.}
\end{equation*}

Suppose $\epsilon = \liminf \epsilon_n > 0$.
Now any point in $B_{\epsilon}(b)$ will have
$b_n(1),\ldots,b_n(\lfloor m/2 \rfloor + 1)$
among its $K$-nearest neighbors per \lemref{lem:Knn}.
As in the original proof, we take two points in $B_{\epsilon}(b) \cap X$ on
either side of $\partial W$ and observe that this forces at least one of the
$K$-nearest neighbors to have a non-modal $f$-value for infinitely many $n$.
This establishes that $B_{\epsilon}(b) \cap X \subseteq Q$.
The rest of the proof proceeds as written, except that we now use the
power of the Second Borel-Cantelli Lemma to show that our chosen events
occur infinitely often, and can therefore meet any $m$.

\thmref{thm:nmc-vor-convergence-support} does not seem to
be easily generalized, given that the $m$ nearest neighbors have no
simply-characterized intersection with the Voronoi neighbors for
$m > 1$.

\section{Conclusion}
\label{sec:conc}

In this work we have articulated a general procedure for selective sampling
for nearest neighbor pattern classification. This procedure is guided by
a selection heuristic $\Phi$, and we proved that the nearest neighbor rule
prediction converges pointwise to the true function on the support of the
domain under each of the following choices for $\Phi$:
distance to sample set (\secref{sec:heu:dist});
non-modal count per Voronoi neighbors (\secref{sec:heu:nmc:vor});
and non-modal count per $K$-nearest neighbors (\secref{sec:heu:nmc:Knn}).
We also established convergence in measure as a corollary, provided that the
domain is separable. Finally, we showed that the first and third alternatives
for $\Phi$ are also valid under the $m$ nearest neighbors rule for any
$m > 1$ (\secref{sec:heu:mNNR}).

There are many avenues for future research; we describe
three open problems that seem particularly interesting and valuable.
\begin{enumerate}
\item For reasons explained in \secref{sec:intro},
  our investigations have taken place in a deterministic
  setting, as opposed to the more general probabilistic approach taken in
  the classical results
  \cite{cover-hart,cover-estimation,pattern-classification}.
  However, it ought to be possible to recover the probabilistic setting if
  the concept of an $f$-contiguous component can somehow be generalized to the
  situation where $f$ is a random variable.
\item Our theorems are silent on the rate of convergence, which is obviously
  an essential question in practical applications. Intuition and
  empirical results (such as Figures~\ref{fig:vor-rnd} and \ref{fig:vor-nmc})
  suggest that selective sampling converges more quickly than iid sampling,
  but we cannot say how much more quickly; nor do we yet understand how the
  prediction error is related to the number of samples.

  Kulkarni and Posner~\cite{kulkarni} derive a number of convergence results
  for arbitrary sampling in a very general setting. Although these results
  have limited direct practical impact, since arbitrary sampling may not
  converge to the true function at all (\secref{sec:prelim:fail}),
  they do include bounds on the expected value of the distance of the
  latest sample from all previous samples, where samples are chosen by a
  stationary process. Intuitively, this value decreases as the predictions
  become more accurate. By comparing the rate of decrease under selective
  sampling versus iid sampling, it may be possible to adduce convergence
  rates, at least in a relative fashion.
\item What are valid choices for the heuristic $\Phi$? We have given three,
  but we do not think this is exhaustive. It also occurs to us that
  the distance to sample set and the two non-modal count heuristics have
  very different convergence proofs---are there perhaps different classes
  of selection heuristics that can be identified? 
\end{enumerate}

Although these and other questions remain open, we nonetheless believe that
the results presented in this paper provide a firm theoretical foundation for
the use of selective sampling techniques in practical applications.

\section*{Acknowledgments}

The authors would like to thank their colleagues at mZeal Communications for
their material and intellectual support for this research. Some of the first
steps in this work were inspired by Lawrence ``David'' Davis and
Stewart W.~Wilson of VGO Associates. Any errors of course belong entirely to
the authors.

\bibliography{convergence-references}{}
\bibliographystyle{plain}

\end{document}